\documentclass[runningheads]{llncs}

\usepackage[utf8]{inputenc} % allow utf-8 input
\usepackage[T1]{fontenc}    % use 8-bit T1 fonts
\usepackage{hyperref}       % hyperlinks
\usepackage{url}            % simple URL typesetting
\usepackage{booktabs}       % professional-quality tables
\usepackage{amsfonts,amsmath,mathtools}       % blackboard math symbols
\usepackage{nicefrac}       % compact symbols for 1/2, etc.
\usepackage{microtype}      % microtypography
\usepackage{lipsum}
\usepackage{graphicx}
\graphicspath{ {./images/} }
\usepackage{comment}
\usepackage{xcolor}
\usepackage[capitalise]{cleveref}
\usepackage{subcaption}
\usepackage{booktabs}
\usepackage{multicol}
\usepackage{multirow}
\usepackage{tabularx}
\crefname{lemma}{Lem.}{Lems.}
\crefname{theorem}{Thm.}{Thms.}
\crefname{definition}{Def.}{Defs.}
\crefname{section}{Sect.}{Sect.}

\newcommand{\mc}{\mathcal }
\newcommand{\Hom}{\text{Hom}}
\newcommand{\id}{\text{Id}}

\begin{document}
\title{Logic Explanation of AI Classifiers by \\Categorical Explaining Functors}

\author{Stefano	Fioravanti\inst{1}
% \orcidID{0000-0001-6918-1805} 
\and Francesco	Giannini\inst{2}
% \orcidID{0000-0001-8492-8110} 
\and Paolo Frazzetto \inst{1}
% \orcidID{0000-0002-3227-0019} 
\and \\
Fabio Zanasi\inst{3}
% \orcidID{0000-0001-6457-1345}
\and Pietro Barbiero\inst{4}
% \orcidID{0000-0003-3155-2564}
}
% \author{}
%
\authorrunning{Fioravanti et al.}
\institute{University of Padova, Italy \and Scuola Normale Superiore, Italy \and  University College London, UK \and Università della Svizzera Italiana, Switzerland}
%
%\author{First Author\inst{1}\orcidID{0000-1111-2222-3333} \and
%Second Author\inst{2,3}\orcidID{1111-2222-3333-4444} \and
%Third Author\inst{3}\orcidID{2222--3333-4444-5555}}
% \institute{}

\begin{comment}
\institute{Princeton University, Princeton NJ 08544, USA \and
Springer Heidelberg, Tiergartenstr. 17, 69121 Heidelberg, Germany
\email{lncs@springer.com}\\
\url{http://www.springer.com/gp/computer-science/lncs} \and
ABC Institute, Rupert-Karls-University Heidelberg, Heidelberg, Germany\\
\email{\{abc,lncs\}@uni-heidelberg.de}}
\end{comment}

%
\maketitle              % typeset the header of the contribution
\begin{abstract} 
The most common methods in explainable artificial intelligence are post-hoc techniques which identify the most relevant features used by pretrained opaque models. Some of the most advanced post hoc methods can generate explanations that account for the mutual interactions of input features in the form of logic rules.
However, these methods frequently fail to guarantee the consistency of the extracted explanations with the model's underlying reasoning. To bridge this gap, we propose a theoretically grounded approach to ensure coherence and fidelity of the extracted explanations, moving beyond the limitations of current heuristic-based approaches. To this end, drawing from category theory, we introduce an explaining functor which structurally preserves logical entailment between the explanation and the opaque model's reasoning. As a proof of concept, we validate the proposed theoretical constructions on a synthetic benchmark verifying how the proposed approach significantly mitigates the generation of contradictory or unfaithful explanations.
\end{abstract}
%%%%%%%%%%%%%%%%%%%%%%%%%%%%5%

% full articles	between 12 and 24 pages (including references) \\
% short articles	between 8 and 12 pages (including references)

\section{Introduction}\label{sec:introduction}

Explainable Artificial Intelligence (XAI) research seeks to understand otherwise opaque AI systems \cite{arrieta2020explainable}. In recent years, XAI methods have been applied in multiple research disciplines %\cite{jimenez2020drug,davies2021advancing,ai4ua,Keskin2023156}.
\cite{jimenez2020drug,davies2021advancing,ai4ua,Keskin2023156,schmidt2009distilling}.
% While extensive research has explored methodologies and conceptual frameworks in XAI~\cite{arrieta2020explainable,palacio2021xai}, many current approaches yield task-specific, ad-hoc explanations, hindering the development of a unified and mathematically rigorous foundation for the field.
%%%%%%%%%%%%%%%%%
Post-hoc techniques~%\cite{palacio2021xai,Ribeiro2018AnchorsHM,ghorbani2019towards}
\cite{palacio2021xai,ribeiro2016model,Ribeiro2018AnchorsHM,ghorbani2019towards}
represent some of the earliest and most common methods as they enable the extraction of explanations from pretrained models. Post-hoc explanations typically aim to rank the input features of an opaque model from the most to the least relevant for the model's reasoning. Some of the most advanced post-hoc approaches go beyond raw feature attribution by considering the mutual interaction between input features in the form of logic rules over human-understandable high-level units of information~\cite{ghorbani2019towards,poeta2023concept},
% Among the different XAI approaches, concept-based methods are particularly interesting as they allow  explanations in terms of human-understandable concepts \cite{ghorbani2019towards,poeta2023concept}. Some interesting XAI methods within this category aim to extract logic rules over these concepts, 
thus enhancing the interpretability of the decision process of the model \cite{barbiero2022entropy,ciravegna2023logic,guidotti2024stable}.
%\cite{barbiero2022entropy,ciravegna2023logic,guidotti2024stable,bobillo2024fuzzy}. 
These approaches attempt to reverse-engineer the model's reasoning, translating opaque neural computations into human-understandable logical formulas. However, the process of extracting discrete rules from continuous models often fails to guarantee logical consistency. 
For instance, let us consider the function $f(x,y)=\min\{1,x+y\}$ that corresponds to the \L ukasiewicz t-conorm in fuzzy logic. %\cite{hajek2013metamathematics}. 
We can easily show that a naive Boolean approximation of this continuous function leads to an inconsistency between the explanation and the original continuous function. Consider observations $(x_1=0.2,y_1=0.2)$ and $(x_2=0.2,y_2=0.4)$ where the model evaluates to $f(x_1,y_1)=0.4$, while $f(x_2,y_2)=0.6$. We can naively extract a logic explanation for these predictions by discretizing the input/output representations (e.g., by thresholding at $0.5$). For the first sample, we obtain $(\bar{x}_1=0, \bar{y}_1=0)$ and $\bar{f(x_1,y_1)}=0$, and for the second sample we obtain $(\bar{x}_2=0, \bar{y}_2=0)$ and $\bar{f(x_1,y_1)}=1$. We refer to this kind of explanation as \emph{inconsistent} w.r.t. the original model, as the same rule $\neg x \land \neg y$ explains opposite class predictions $\bar{f(x_1,y_1)}=1 \neq \bar{f(x_2,y_2)}=0$. In other words, the extracted rules may contradict each other or fail to faithfully reflect the model's behavior, leading to misleading or unreliable explanations.  Even worse, in deep learning the models are composed of multiple layers or modules. Hence, a well-fitting explanation of deep neural networks should guarantee that composing the different explanations extracted from the different layers produces an overall explanation that is consistent w.r.t. explaining the model as a whole. This is often not the case with existing approaches \cite{ciravegna2023logic,dominici2024causal}. 
These discrepancies imply that current post-hoc explainability methods do not respect functional composition, meaning that explanations derived from smaller, interpretable steps may not be reflected by the behavior of the composed function. 
This motivates our work to provide a mathematically solid theory of XAI methods that ensures coherence over composition and fidelity of the extracted explanations, moving beyond the limitations of current heuristic-based approaches. 

% \textbf{Our Framework.} 
To achieve this goal, we rely on Category Theory as a mathematical field designed to deal with \textit{processes} and their composition, thereby providing a robust mathematical framework to study compositionality between explainers and classifiers. Indeed, category theory has been widely applied in theoretical computer science~\cite{abramsky2004categorical,stein2021}, and more recently in AI~\cite{cruttwell2022categorical,cat2024}. %ong2022learnable
More specifically, in this paper we introduce the notion of ``explaining functor'', a mapping between continuous functions representing concept-based neural networks into Boolean explanations. By defining objects as fuzzy spaces and morphisms as explanatory transformations, our functor structurally preserves logical entailment, thus mitigating the prevalent issue of contradictory or unfaithful explanations. This approach ensures that explanations are not only human-readable, but also logically sound and consistent with the model's internal workings.
As a proof-of-concept, we show how the proposed theoretical constructions can be used experimentally on synthetic benchmarks based on logic operators. The experiments demonstrate in practice how the theoretical framework can be used. 
\textbf{Our contributions} can be summarized as follows: (i) we identify categories of functions whose boolean explanations are consistent and combinable by design; (ii) we define and study different categorical functors that associate logic formulas to concept-based fuzzy functions; (iii) we show in practical examples how the usage of these functors yields coherent explanations.

\section{Preliminaries of Category Theory}
A \emph{category} \(\mc{C}\) consists of a collection of \emph{objects} $\mc{C}^o$ and a collection of \textit{morphisms} $\mc{C}^h$ between pairs of objects. We write $f \colon X \to Y$ for a morphism from object $X$ to object $Y$, and $\Hom(X,Y)$ for the set of such morphisms. For all morphisms $f \colon X \to Y$ and $g \colon Y \to Z$ we require the existence of a \emph{composite} morphism $g \circ f \colon X \to Z$, with composition being associative. Moreover, for each $X \in \mc{C}^o$, we require the existence of an \emph{identity} morphism $\id_X \in \Hom(X,X)$ satisfying $f \circ \id_X = f = \id_Y \circ f$. Compared to sets, categories are suitable for studying groups of objects that maintain a particular structure, along with the transformations that preserve this structure. A basic example of this notion is the category whose objects are vector spaces and morphisms are linear maps. %In general, any class of general algebras of the same type with their relative homomorphism forms a category. 
    
As examples of two categories central to our development, we mention the \emph{category of Boolean Functions} $\mc{B}=(\mc{B}^{o},\mc{B}^{h})$ where $\mc{B}^{o}=\{\{0,1\}^n\mid\ n\in\mathbb{N}\}$ with $\mbox{Hom}_{\mc{B}}(X,Y) =\{f\colon\ X \to Y \mid f \text{ is a function}\}$, for all $X, Y \in \mc{B}^{o}$,
and the \emph{category of Fuzzy Functions} $\mc{F}=(\mc{F}^{o},\mc{F}^{h})$ where $\mc{F}^{o}=\{[0,1]^n\mid\ n\in\mathbb{N}\}$ with, $\mbox{Hom}_{\mc{F}}(X,Y) =\{f\colon\ X \to Y \mid f \text{ is a function}\}$, for all $X, Y \in \mc{F}^{o}$.

Category theory allows us to model not only maps between objects within a given category, but also transformations involving categories themselves. A \emph{functor} models the notion of structure-preserving map between categories. In particular, given two categories $\mc{C}=(\mc{C}^{o},\mc{C}^{h})$ and $\mc{D}=(\mc{D}^{o},\mc{D}^{h})$, a \emph{functor} $F \colon \mc{C} \to \mc{D}$ is a map assigning to each object $X \in \mc{C}^o$ an object $F(X) \in \mc{D}^o$, and to each $f \colon X \to Y$ in $\mc{C}^h$ a morphism $F(f) \colon F(X) \to F(Y)$ in $\mc{D}^h$, satisfying $F(Id_X) = Id_{F(X)}$, for all $X \in \mc{B}^{o}$, and $F(f \circ g) = F(f) \circ F(g)$.
The latter is known as \emph{compositionality} and guarantees that the mapping of compound morphisms is equivalent to the composition of ``simpler'' morphisms.
% Preservation of composition is a particularly relevant property for practical purposes: it amounts to saying that we may compute the value of $F$ on a morphism $h$ entirely in terms of the values of $F$ on $h_1, h_2, \dots, h_k$, where $h_1 \circ h_2 \circ \dots \circ h_k = h$ is some decomposition of $h$ into `simpler' morphisms. This property is known as \emph{compositionality} and allows to scale the complexity of the transformation captured by $F$.

\section{Explaining Functor from $\delta$-COH to Boolean Functions}

Even simple fuzzy functions are inconsistent when projected into Boolean functions; thus we introduce a special class of fuzzy functions, the $\delta$-COH functions, which is composed by functions with coherent Boolean explanations (Boolean logic rules). We prove that this class can be equipped with a categorical structure (\cref{thm:deltacorcat}), being closed under composition. Furthermore, we define a functor that maps this category to the category of Boolean sets with Boolean functions (\cref{th:booleanfunctor}), thereby providing a rigorous mathematical framework for their interpretation and transformation.
 
\subsection{Category of $\delta$-coherent functions}

Let $\delta \colon [0,1] \rightarrow S$, where $S\subseteq [0,1]$. The $\delta$ is called a \emph{projection}\label{def:delta-proj} over $S$ if $\delta \circ \delta = \delta$, i.e., $\delta(x) = x$ for all $x \in S$ (idempotency). Then, the concept of projection onto smaller spaces embodies a fundamental aspect of explainability: the transformation of complex expressions defined over large domains into simpler representations within a more restricted set. %We also notice that the theory developed in the following sections can be completely generalized outside $[0,1]$. % PAOLO: Omesso, in queanto viene detto già a fine sezione
From now on, we will use the same notation for $\delta$ and for all its component-wise applications, i.e., $\delta(x) = [\delta(x_1),\dots,\delta(x_n)]$, for all $x \in [0,1]^n$.

\begin{definition}\label{def:delta-funct}
Let $\delta $ be a projection. A function $f \colon [0,1]^n \rightarrow [0,1]^m$ is said to be \textit{$\delta$-coherent on $x \in [0,1]^n$}, if $\delta \circ f(x)  = \delta  \circ f \circ \delta(x)$ i.e., if $\delta(f(x)) = \delta(f(\delta (x)))$. We denote by $\delta$-COH$(f)=\{x\in [0,1] \mid\ f\mbox{ is }\delta\mbox{-coherent on }x\}$. 
% Furthermore, $f$ is said to be $\delta$-\emph{coherent} if it is $\delta$-\emph{coherent} on every $x\in [0,1]^n$. 
$f$ is said $\delta$-\emph{coherent} if $\delta$-COH$(f)=[0,1]$.
\end{definition}
We notice that if a function $f\colon [0,1]^n\to [0,1]^m$ is $\delta$-coherent on a certain $x\in[0,1]^n$, then it is also $\delta$-coherent with every $y\in[0,1]^n$ such that $\delta(x)=\delta(y)$. % Given a fixed $\delta \colon [0,1] \rightarrow S$, we will often denote $\delta(x)$ by $x_{\delta}$. Furthermore, for a given $f\colon [0,1]^n \rightarrow [0,1]^m$ we denote by $f^\delta$ the function $f^\delta \colon x \mapsto \delta(f(x))$. 

\noindent An example of \cref{def:delta-funct} is given by Boolean threshold functions. In particular, let $\alpha \in [0,1]$, then we call $\delta_{\alpha}\colon[0,1]\to \{0,1\}$ $\alpha$-\emph{booleanization} if $\delta_{\alpha}(x) = 1 \text{ if and only if }x \geq \alpha,\mbox{ and }0\mbox{ otherwise}$, for all $x \in [0,1]$.

\begin{lemma}
\label{lem:constant}
   Every constant fuzzy function $f$ is $\delta$-coherent, for any projection $\delta$. 
\end{lemma}

\noindent The proof of the previous lemma is trivial. We can see that we can easily build examples of functions satisfying or not \cref{def:delta-funct}.
Let us consider the $\alpha$-booleanization function $\delta_{0.5}$ where $\alpha=0.5$. Then trivially the minimum t-norm $\min(x,y)$ is $\delta$-coherent, indeed $\delta(\min(x,y))=\delta(\min(\delta(x),\delta(y))$.
However, this is not the case for the \L ukasiewicz t-norm $x\otimes y = \max(0,x+y-1)$, for example, by taking $x=y=0.6$, we have $\delta(0.6,0.6)=(1,1)=\delta(1,1)$ while $\delta(0.6\otimes 0.6)=\delta(0.2)=0\neq1=\delta(1)=\delta(1\otimes 1)$.

 % In particular, in this paper we are interested in $\delta$-\emph{coherent} functions for some $\delta_{\alpha}$.
% For simplicity, in the following we will assume $\alpha$ is fixed omitting the $\alpha$ dependence, and we denote by $\overline{x}$ the $\alpha$-booleanization of $x$, i.e. $\overline{x} = \delta_{\alpha}(x)$. In case $x\in[0,1]^n$, $\overline{x}\in\{0,1\}^n$ denotes $\delta_{\alpha}$ applied component-wise to $x$. 

\begin{theorem}\label{thm:deltacorcat}
Let $\delta \colon [0,1] \rightarrow S \subseteq [0,1]$. The pair $\mc{C}_\delta=(\mc{C}_\delta^{o},\mc{C}_\delta^{h})$ where
$\mc{C}_\delta^{o}=\{[0,1]^n \mid \ n\in\mathbb{N}\}$ and for all $X,Y \in \mc{C}_\delta^{o}$, $\text{Hom}_{\mc{C}_\delta}(X,Y) =\{f\colon\ X \to Y \mid f\mbox{ is } \delta-\text{coherent}\}$ is a category called category of \textit{$\delta$-coherent functions} $\delta$-COH.
\end{theorem}
\begin{proof}

% By hypothesis $\delta \in [0,1]$. 
For every $X=[0,1]^n\in\mc C_{\delta}^o$, the identity morphism $1_X$ is $\delta$-coherent by definition, and thus $1_{X}\in \mbox{Hom}_{\mc{C}_\delta}(X,X)$. Assuming $f \in \mbox{Hom}_{\mc{C}_\delta}([0,1]^n,[0,1]^m),g \in \mbox{Hom}_{\mc{C}_\delta}([0,1]^m,[0,1]^k)$, we prove that $g\circ f\in \mbox{Hom}_{\mc{C}_\delta}([0,1]^n,[0,1]^k)$. For this purpose, it suffices to show that $g\circ f$ is $\delta$-coherent. We can see that $\delta \circ f \circ \delta = \delta \circ f$ and $\delta \circ g \circ \delta = \delta \circ g$ since $f$ and $g$ are $\delta$-coherent and, in turn, $\delta \circ f \circ g \circ \delta = \delta \circ f \circ \delta \circ g \circ \delta =  \delta \circ f \circ \delta \circ g = \delta \circ f \circ g$ and the claim holds.
\end{proof}
    
\begin{remark}
    We notice that we get a different category for any given projection $\delta$.
\end{remark}

\subsection{Functor from $\delta$-coherent functions to Boolean functions}
For each Boolean function $f:\{0,1\}^n\rightarrow\{0,1\}^m$, we can consider $m$ Boolean formulas $\varphi_{f_i}$, whose truth-values are defined by $f_i$, where $i \in \{1,\ldots,m\}$. As Boolean functions correspond to logic formulas, they are interpretable by design. Now we will see how we associate to each $\delta$-coherent function its Boolean explanation by using a categorical functor that we call the \textit{explaining functor}.
In particular, we define a functor from $\mc{C}_{\delta}$ to $\mc{B}$ for any given projection $\delta \colon [0,1] \rightarrow \{0,1\}$.

\begin{definition}
\label{def:functor}
    Let $\delta$ be a projection and $F_\delta:\mc{C}_\delta\to \mc{B}$ be defined by:
    \begin{itemize}
        \item[(1)] for every $X=[0,1]^n\in \mc{C}_{\delta}^{o}$, $F_\delta(X)=\{0,1\}^n\in \mc{B}^{o}$;
        \item[(2)] for all $f\in\text{Hom}_{\mc{C}_\delta}([0,1]^n,[0,1]^m)$, we define $F_{\delta}(f)\in\text{Hom}_{\mc{B}}$ $(\{0,1\}^n,\{0,1\}^m)$, such that
        $F_{\delta}(f) = \delta \circ f$.
    \end{itemize}
\end{definition}
In the following lemma and theorem, whose proof is straightforward, we denote $F_{\delta}(f)$ by $f^\delta$ and call $f^\delta$ the \emph{$\delta$-function} of $f$.

\begin{lemma}
    \label{th:boole_centroid}
    Let $\delta$ be a projection and let $f\in\Hom_{\mc{C}_{\delta}}([0,1]^n,[0,1]^m)$. Then. for all $x\in[0,1]^n$, we have: $\delta \circ f=\delta \circ f \circ \delta=f^{\delta} \circ \delta.$
\end{lemma}
\begin{comment}
\begin{proof}
    It follows from the fact that $f$ is $\delta$-coherent and $\delta \circ \delta=\delta$.
\end{proof}
\end{comment}

\begin{theorem}\label{th:booleanfunctor}
    Let $\delta$ be a projection. Then $F_\delta\colon\mc{C}_{\delta} \rightarrow\mc{B}$ (\cref{def:functor}) is a functor.
\end{theorem}
\begin{comment}
\begin{proof}
First of all, we have to prove that for every object $[0,1]^n = A \in \mc{C}^o_{\delta}$, $F_\delta(\id_{A})=\id_{A}^{\delta}=\id_{F_\delta(A)}$. Indeed, $\id^\delta_{A}=\delta \circ \id_{A}=\delta_{|A} = \id_{F_\delta(A)}$. 
        Furthermore, let $f \in \mbox{Hom}_{\mc{C}_\delta}([0,1]^n,[0,1]^m)$ and $g \in \mbox{Hom}_{\mc{C}_\delta}([0,1]^m,[0,1]^k)$  we have to prove that $(g\circ f)^{\delta}=g^\delta\circ f^\delta$. % and hence that $g \circ f \in  \mbox{Hom}_{\mc{C}_\delta}([0,1]^n,[0,1]^k)$. 
         By \Cref{th:boole_centroid} we have
        $\delta \circ g = g^\delta \circ \delta$, and hence
        \[
        (g\circ f)^\delta=\delta \circ g \circ f = \delta \circ g \circ \delta \circ f = g^\delta \circ \delta \circ f= g^\delta \circ f^\delta.
        \]
\end{proof}
\end{comment}

\section{Extending the Explaining Functor}
% \subsection{Logic-based Explaining Functor}

This section provides a method for extending the categorical structure that equips $\delta$-coherent fuzzy functions to all fuzzy functions. As we saw, not all fuzzy functions are $\delta$-coherent, however this assumption is fundamental to get a functor to the category $\mc{B}$. For instance, removing the assumption of $\delta$-coherency, it is easy to produce a function not satisfying \cref{th:booleanfunctor}. For example, let $\delta_\alpha\colon [0,1] \rightarrow \{0,1\}$ an $\alpha$-booleanization with $\alpha=0.5$, let $g\colon [0,1] \rightarrow [0,1]$ with $g(0) = 0$ and $g(x) = 1$ for all $x >0$, and let $f\colon [0,1] \rightarrow [0,1]$ such that $f(x) = 0.2$ if $x <0.5$ and $f(x) = 1$ otherwise. Then:
    \[
     (g \circ f)^{\delta}(0) = \delta(g(f(0))) = \delta(g(0.2)) = 1 \not= 0 = \delta(0) = \delta(g( \delta(f(0)))) =
     \\g^{\delta} \circ f^{\delta}(0).   
    \]
In particular, we can see that $g^{\delta} \circ f^{\delta} = \id_{\{0,1\}}$ while $(g \circ f)^{\delta}$ is the constant~$1$ function. Note that $g$ is not $\delta_{\alpha}$-coherent since $\delta(g(0.2)) \not= \delta(g(\delta(0.2))$. 
This example shows how easy it is to find misleading explanations as the composition of two different ones. For this reason, we are interested in defining a mapping that associates to each fuzzy function a $\delta$-coherent function, by preserving the functor compositionality as obtained above.
However, transforming from a non-$\delta$-coherent to a $\delta$-coherent function may require a non-negligible modification of the original function, which can be carried out in many ways.  
In this paper we focus on two main kinds of approaches to make $\delta$-coherent a non-$\delta$-coherent function:
1) ``correct'' the input, by expanding the domain to disambiguate input points witnessing a failure of $\delta$-coherency with respect to their $\delta$-functions, 2) ``correct'' the output on inputs that witness a failure of $\delta$-coherency.

\subsection{From Fuzzy to $\delta$-Coherent Functions: Naive}

\textbf{1) Domain extension.} 
The next theorem provides a criterion to associate a coherent function to a non-coherent one via domain extension. %generalizes the result to $m$-ary functions, with $m \in \mathbb{N}$. 
Note that we only need to extend the domain for components that are non-$\delta$-coherent.

% , we introduce the concept \emph{non}-\emph{functional} components. Let $f:[0,1]^n\to [0,1]^m$ and let $1 \leq i \leq m$. Then $i$th component of $f$ is a \emph{non}-\emph{functional} component if $\overline{f}_i$ is not a function, i.e. if there exist $x, y \in [0,1]^n$ such that $\overline{x} = \overline{y}$ and $\overline{f(x)} \not= \overline{f(y)}$.

\begin{theorem}
\label{teo:extcoh}
    Let $\delta$ be a projection, let $f:[0,1]^n\to [0,1]^m$, and let $p\leq m$ be the number of components of $f$ that are not $\delta$-coherent. 
    Then $f$ can be extended to a function $\tilde{f}:[0,1]^{n + p}\to [0,1]^m$ that is $\delta$-coherent and such that $\forall i:\ 1 \leq i \leq m$:
    \[
\forall x \in[0,1]^n, \forall c\in[0,1]^p,  \qquad \tilde{f}_i(x,c)  =  f_i(x) \quad \mbox{ if }f_i\mbox{ is }\delta\mbox{-coherent on }x.
\]
%%%%%STEFANO PROOF
%     \begin{proof}
%    \[
%    \tilde{f}_i(x,c)=
%    \begin{cases}
%        f_i(x) & \mbox{ if }f\mbox{ is }\delta\mbox{-coherent on }x\\
%        c_i & \mbox{otherwise.}
%    \end{cases}
%    \]  
% Furthermore, we define $\tilde{f}_i(x,c) = f_i(x)$ for the remaining indices.
\end{theorem}
    \begin{proof}
    Let $W = \{i_1,\ldots,i_p\} \subseteq\{1,\ldots,m\}$ be the set of indices of components of $f$ that are not $\delta$-coherent. 
For every $x\in [0,1]^n, c \in[0,1]^p$, we define $\tilde{f}_i(x,c) = f_i(x)$ if $i\notin W$, whereas if $i_j\in W$:
   \[
   \tilde{f}_{i_j}(x,c)=
   \begin{cases}
       f_{i_j}(x) & \mbox{ if }f_{i_j}\mbox{ is }\delta\mbox{-coherent on }x\\
       c_j & \mbox{otherwise}
   \end{cases}
   \]  
   %where $j_i\in\{1,\ldots,p\}$.
Then clearly $\tilde{f}$ extends $f$ as specified in the statement. In addition, let $x\in [0,1]^n, c\in[0,1]^p$. By definition of $\tilde{f}$, for all $i\in \{1,\ldots,m\}$, 
either $f_i$ is $\delta$-coherent on $x$ or is constant, and therefore $\delta \circ \tilde{f}_i(x, c) = \delta \circ f_i(x) =  \delta \circ f_i \circ \delta(x)=\delta \circ \tilde{f}_i \circ \delta(x,c)$, or $\delta \circ \tilde{f}_i(x, c) = \delta \circ c_{j}(x) = \delta \circ c_{j} \circ \delta (x) =
\delta \circ \tilde{f}_i \circ \delta(x, c)$ for any $j \in \{1,\ldots,p\}$. 
\end{proof}

\noindent Basically, the idea behind the expansion $\tilde{f}$ is to ``correct'' the components that are not $\delta$-coherent by adding one input per each of them.

\medskip
\textbf{2) Output Modification.} %The next theorem provides a criterion to associate a coherent function to a non-coherent one via output modification.
\begin{theorem}\label{teo:cohgi}
     Let $\delta$ be a projection, $f,g\colon [0,1]^n \rightarrow [0,1]^m$, with $g$ $\delta$-coherent. The function $\hat{f}\colon [0,1]^n \rightarrow [0,1]^m$ defined by:

    \begin{equation*}\label{eq:fgcoh}
       \hat{f}_{i}(x)=
   \begin{cases}
       f_{i}(x) & \mbox{ if }f_{i}\mbox{ is }\delta\mbox{-coherent on }x\\
       g_{i}(x) & \mbox{otherwise}
   \end{cases}
       \end{equation*} 
    for all $x \in [0,1]^n$ and $i=\{1,\dots,m\}$ is $\delta$-coherent.
\end{theorem}

\begin{proof}
    Let $x \in [0,1]^n$. Then for all $i\in\{1,\ldots,m\}$, if $f_i$ is coherent on $x$, then $\delta \circ \hat{f}_i(x) = \delta \circ f_i \circ \delta(x) = \delta \circ \hat{f}_i \circ \delta(x)$; otherwise $\delta \circ \hat{f}_i(x) = \delta \circ g_i(x) = \delta \circ g_i \circ \delta(x) = \delta \circ \hat{f}_i \circ \delta(x)$, since $g_i$ is $\delta$-coherent.
\end{proof}

% \ste{Fra check this}
% A special case of this approach to make a function $\delta$-coherent was already successfully used in \cite{barbiero2021lens}. In this case the explaining functions where constrained to be $\delta $-coherent by  defining $g_i$ to be constantly equal to $1$, demonstrating a gaining of fidelity and accuracy of the model explainer forcing the $\delta$-coherency with respect to an $\alpha$-booleanization function.

\textbf{General Case.} 
% These approaches can be seen as special instances of a general case by making explicit the mapping associating a $\delta$-coherent to a fuzzy function.
\begin{definition}\label{def:gendeltcon}
    Let $\delta$ be a projection. A $\delta$-\emph{coherency function} is a map $\Gamma$ from the set $\mc{F}^h = \{f \colon [0,1]^n \rightarrow [0,1]^m\mid n,m \in \mathbb{N}\}$ to its subset of $\delta$-coherent functions $\mc{C}_\delta^h$, such that $\Gamma(f) = f$ (idempotency)
    for every $f$  that is $\delta$-coherent.
\end{definition}

\noindent Notice that requiring idempotency over $\delta$-coherent functions is fundamental to define a well-suited $\delta$ coherent function. 
Clearly, the cases of \textit{domain extension} and \textit{output modifications} lie under the scope of \cref{def:gendeltcon}, however, they are intrinsically not compositional, i.e., they are not well behaved w.r.t. functional composition. We will show an easy example of this phenomenon.

\begin{theorem}\label{lem:non-comp}
    Let $\delta$ be a projection, and let $\Gamma$ be a $\delta$-coherency function. Then $\Gamma$ is not compositional, that is, there exists $f,g \in \mathcal{F}^h$ such that $\Gamma(g \circ f) \not= \Gamma(g) \circ \Gamma(f)$.
\end{theorem}

\begin{proof}
Let $g\colon[0,1] \rightarrow [0,1]$ be a non-$\delta$-coherent function. We have two cases:

Case domain of $\text{Dom}(\Gamma(g)) \not= [0,1]$. The claim holds by taking $f\colon[0,1] \rightarrow [0,1]$ as the constant $0$. In fact, $\Gamma(g \circ f)$ is defined while $\Gamma(g) \circ \Gamma(f)$ is not defined. 
     
Case $\text{Dom}(\Gamma(g))=[0,1]$. Let $S \subset [0,1]$ be the subset of points in the domain of $g$, where $g$ is not $\delta$-coherent. Since $g$ is not $\delta$-coherent, $|S| > 0$. 
Furthermore, there exist $a,b\in S$, such that $\delta(a)=\delta(b)$ with $\delta(g(a))\neq \delta(g(b))$ and $\delta(\Gamma(g)(a))= \delta(\Gamma(g)(b))$, since $\Gamma(g)$ is $\delta$-coherent by \cref{def:gendeltcon}. Hence, either $\Gamma(g)(a)\neq g(a)$ or $\Gamma(g)(b)\neq g(b)$. Let us assume, w.l.o.g., $\Gamma(g)(a)\neq g(a)$.
Let $f\colon[0,1] \rightarrow [0,1]$ be the $\delta$-coherent function constantly equal to $a$. Then clearly $g \circ f(x) = g(a)$, for all $x \in [0,1]$, and thus $g \circ f$ is $\delta$-coherent (cf. \cref{lem:constant}). Then, by considering $a$, we have $\Gamma(g \circ f)(a)=(g\circ f)(a)=g(a)\neq \Gamma(g)(a)=(\Gamma(g) \circ f)(a)=(\Gamma(g) \circ \Gamma(f))(a)$. 
Hence $\Gamma$ is not compositional. 

\end{proof}

\subsection{From Fuzzy to $\delta$-Coherent Functions}

From \cref{lem:non-comp} we see that witnessing of non-compositionality can be constructed for any given $\delta$-coherency function $\Gamma$. This fact creates an obstacle to extend the explaining functor to non-$\delta$-coherent functions. With this aim we will define a new category that allows to compositionally associate a $\delta$-coherent function to any given fuzzy function. To do so, we have to first introduce an equivalence relation over the set of all fuzzy functions.

\begin{definition}
    Let $\delta$ be a projection, $\Gamma$ be a $\delta$-coherency function, and $\mathcal{F}^h$ be the set of fuzzy functions. We denote by $\equiv_\Gamma$ the binary relation over $\mathcal{F}^h$ defined by:
    \begin{equation*}
        f \equiv_\Gamma g \qquad \text{ iff } \qquad \Gamma(f) = \Gamma(g).
    \end{equation*}
\end{definition}

\begin{lemma}\label{lem:cprop}
     Let $\delta$ be a projection, and let $\Gamma$ be a $\delta$-coherency function. Then:
     \begin{enumerate}
         \item[(1)] $\equiv_\Gamma$ is an equivalence relation;
         \item[(2)] there exists a unique $\delta$-coherent function $f$ in every equivalence class of $\equiv_\Gamma$.
     \end{enumerate}
\end{lemma}
\begin{comment}
\begin{proof}
    First we can observe that $(1)$ follows from the fact that the equality is an equivalence relation. For $(2)$ suppose that $f$ and $g$ are $\delta$-coherent functions in the same equivalence class, i.e., $\Gamma(f) = \Gamma(g)$. By \cref{def:gendeltcon}, $f = \Gamma(f) = \Gamma(g) = g$. 
\end{proof}
\end{comment}

\noindent The proof is standard. With this tool, we can deal with the compositonality problem of non-$\delta$-coherent functions. To this aim, we define the following category.

\begin{theorem}
\label{th:cateqcla}
Let $\delta$ be a projection, and let $\Gamma$ be a $\delta$-coherency function. The pair $\mc{C}_{(\delta,\Gamma)}=(\mc{C}_{(\delta,\Gamma)}^{o},\mc{C}_{(\delta,\Gamma)}^{h})$ where $\mc{C}_{(\delta,\Gamma)}^{o}=\{[0,1]^n \mid \ n\in\mathbb{N}\}$ and:
\begin{enumerate}
    \item for all $A,B \in \mc{C}_{(\delta,\Gamma)}^{o}$, $\text{Hom}_{\mc{C}_{(\delta,\Gamma)}}(A,B) =\{[f]_{\equiv_\Gamma}\mid  \Gamma(f) \colon A \to B\mid f\in\mc{F}^h\}$, where $[f]_{\equiv_\Gamma}$ is the $\equiv_\Gamma$-equivalence class of $f$;
    \item for all $A,B,C \in \mc{C}_{(\delta,\Gamma)}^{o}$, $[f]_{\equiv_\Gamma} \in \text{Hom}_{\mc{C}_{(\delta,\Gamma)}}(A,B)$, and $[g]_{\equiv_\Gamma} \in \text{Hom}_{\mc{C}_{(\delta,\Gamma)}}(B,C)$, we define $[g]_{\equiv_\Gamma} \circ [f]_{\equiv_\Gamma} = [\Gamma(g) \circ \Gamma(f)]_{\equiv_\Gamma}$;
\end{enumerate}
forms a category that we call the category of \textit{quotient fuzzy functions}.
\end{theorem}

\begin{proof}
    First we observe that $\circ$ is well-defined since, by \cref{def:gendeltcon} and \cref{lem:cprop}, for all $[f]_{\equiv_\Gamma} \in \text{Hom}_{\mc{C}_{(\delta,\Gamma)}}(A,B)$, $[g]_{\equiv_\Gamma} \in \text{Hom}_{\mc{C}_{(\delta,\Gamma)}}(B,C)$, and for all $f' \in [f]_{\equiv_\Gamma}$ and $g' \in [g]_{\equiv_\Gamma}$, $\Gamma(f)$ and $\Gamma(g)$ are the unique $\delta$-coherent functions in $[f]_{\equiv_\Gamma}$ and $[g]_{\equiv_\Gamma}$ respectively, and thus 
    $[g]_{\equiv_\Gamma} \circ [f]_{\equiv_\Gamma} = [\Gamma(g) \circ \Gamma(f)]_{\equiv_\Gamma} = [\Gamma(g') \circ \Gamma(f')]_{\equiv_\Gamma}=[g']_{\equiv_\Gamma} \circ [f']_{\equiv_\Gamma}$. Let us consider $X=[0,1]^n\in\mc C_{\delta}^o$, then $[\id_X]_{\equiv_\Gamma}$, the $\equiv_\Gamma$-class of the identity over $X$, is the identity morphism since $\id_X$ is the $\delta$-coherent function in $[\id_X]_{\equiv_\Gamma}$. Thus $[\id_X]_{\equiv_\Gamma} \in \mbox{Hom}_{\mc{C}_{(\delta,\Gamma)}}(X,X)$. Let us assume $[f]_{\equiv_\Gamma} \in \mbox{Hom}_{\mc{C}_{(\delta,\Gamma)}}([0,1]^n,[0,1]^m)$ and $[g]_{\equiv_\Gamma} \in \mbox{Hom}_{\mc{C}_{(\delta,\Gamma)}}([0,1]^m,[0,1]^k)$. Then we have to prove that $[g]_{\equiv_\Gamma}\circ [f]_{\equiv_\Gamma} \in \mbox{Hom}_{\mc{C}_{(\delta,\Gamma)}}([0,1]^n,[0,1]^k)$. For this purpose, it is enough to show that $\Gamma(g) \circ \Gamma(f)$ is $\delta$-coherent, but this was already observed in the proof of \cref{thm:deltacorcat}. Thus $[g]_{\equiv_\Gamma}\circ [f]_{\equiv_\Gamma} = [\Gamma(g) \circ \Gamma(f)]_{\equiv_\Gamma} \in  \mbox{Hom}_{\mc{C}_{(\delta,\Gamma)}}([0,1]^n,[0,1]^k)$.
\end{proof}

\noindent We can now define a functor from category $\mc{C}_{(\delta,\Gamma)}$ to category $\mc{C}_\delta$.

\begin{definition}
\label{def:functor2}
Let $\delta \colon [0,1] \rightarrow \{0,1\}$ be a projection, $\Gamma$ a $\delta$-coherency function and 
 $F_\Gamma:\mc{C}_{(\delta,\Gamma)}\to \mc{C}_{\delta}$ the mapping defined by:
    \begin{itemize}
        \item[(1)] for every $X=[0,1]^n\in \mc{C}_{(\delta,\Gamma)}^{o}$, $F_\Gamma(X)=[0,1]^n\in \mc{C}_{\delta}^{o}$, i.e. $F_\Gamma(X)=X$;
        \item[(2)] for all $[f]_{\equiv_\Gamma}\in \mbox{Hom}_{\mc{C}_{(\delta,\Gamma)}}([0,1]^n,[0,1]^k)$ we define $F_\Gamma([f]_{\equiv_\Gamma})=\Gamma(f)\in \mc{C}_{\delta}^h$.
    \end{itemize}
\end{definition}

\begin{theorem}\label{th:quotient functor}
    $F_\Gamma$ introduced in \cref{def:functor2} is a functor from $\mc{C}_{(\delta,\Gamma)}$ to $\mc{C}_{\delta}$.
\end{theorem}

\begin{proof}
First, we have to prove that for every object $A \in \mc{C}^o_{(\delta,\Gamma)}$, $F_\Gamma([\id_{A}]_{\equiv_\Gamma})=\id_{F_\Gamma(A)}=\id_A$, which is trivial as $F_\Gamma([\id_{A}]_{\equiv_\Gamma})=\Gamma(\id_{A})=\id_A$, as $\id_A$ is $\delta$-coherent.
Secondly, let $[f]_{\equiv_\Gamma} \in \mbox{Hom}_{\mc{C}_{(\delta,\Gamma)}}([0,1]^n,[0,1]^m)$ and $[g]_{\equiv_\Gamma} \in \mbox{Hom}_{\mc{C}_{(\delta,\Gamma)}}$ $([0,1]^m,[0,1]^k)$  we have to prove that $F_\Gamma([g]_{\equiv_\Gamma}\circ [f]_{\equiv_\Gamma})=F_\Gamma([g]_{\equiv_\Gamma})\circ F_\Gamma([f]_{\equiv_\Gamma})$. However, this follows from \cref{th:cateqcla}.

\end{proof}

\noindent Notice that, as the composition of two functors is a functor, trivially by composing the functor $F_\Gamma$ and $F_\delta$, we get a functor from $\mc{C}_{(\delta,\Gamma)}$ to $\mc{B}$. This functor, which we indicate by $F_{(\delta,\Gamma)}$ is characterized as follows.

\begin{theorem}

Let $\delta \colon [0,1] \rightarrow \{0,1\}$ be a projection and $\Gamma$ a $\delta$-coherency function.
Then the mapping $F_{(\delta,\Gamma)}:\mc{C}_{(\delta,\Gamma)}\to \mc{B}$ defined by:
\begin{itemize}
        \item[(1)] for every $X=[0,1]^n\in \mc{C}_{(\delta,\Gamma)}^{o}$, $F_{(\delta,\Gamma)}(X)=\{0,1\}^n\in \mc{B}^o$;
        \item[(2)] for every morphism $[f]_{\equiv_\Gamma}\in \mbox{Hom}_{\mc{C}_{(\delta,\Gamma)}}([0,1]^n,[0,1]^m)$ the function $F_{(\delta,\Gamma)}(f)=\Gamma(f)^\delta\in \mc{B}^h$, with $\Gamma(f)^\delta\colon\{0,1\}^n\to \{0,1\}^m$ such that
        $\Gamma(f)^\delta =\delta \circ \Gamma(f)$
\end{itemize}
 is a functor from $\mc{C}_{(\delta,\Gamma)}$ to $\mc{B}$, and it satisfies $F_{(\delta,\Gamma)}=F_\delta\circ F_\Gamma$.
\end{theorem}
%\begin{proof}
%The proof follows from \cref{th:booleanfunctor} and \cref{th:quotient functor}.
%\end{proof}

The results presented in this section can be fully generalized to arbitrary sets beyond $[0,1]$ and $\{0,1\}$, thereby allowing the use of different types of logic as targets for a coherent explainer. Furthermore, the functions $\Gamma$ and $\delta$ are subject to no additional constraints apart from the idempotency conditions on $\delta$-coherent functions and the image of $\delta$. Specifically, $\Gamma(f) = f$ for any $\delta$-coherent $f$ and $\delta(x) = x$ for all $x \in \delta([0,1])$. This demonstrates that our approach possesses significant flexibility and can be uniformly applied to explainers operating on non-Boolean logics.

% \ste{togliamo l'ambiente paragraph per risparmiare spazio?}
\paragraph{Explaining Fuzzy Functions by Categorial Functors.}

As we noticed, it is not possible to define a functor from $\mc{F}$ to $\mc{B}$. However, we can exploit the explaining functors defined above to associate to each fuzzy function a Boolean function. This can be done according to the following steps: 

\begin{enumerate}
    \item Let us consider a fuzzy function $f\in\mc{F}^h$, with $f:[0,1]^n\rightarrow[0,1]^m$, a projection $\delta$ and a $\delta$-coherency function $\Gamma$.
    \item We exploit the functor $F_\Gamma$ that associates to $f$ its unique $\delta$-coherent function $\Gamma(f)\in\mc{C}_\delta^h$ in its equivalence class $[f]_{\equiv_\Gamma}$.
    \item Then, by means of the functor $F_\delta$ we obtain the Boolean function $\Gamma(f)^\delta\in\mc{B}^h$.
    \item Finally, each component of $\Gamma(f)^\delta$ can be written as a Boolean formula by converting its truth-table into a disjunctive logic clause.
    \end{enumerate}

\section{Experimental Analysis}

To demonstrate the practical validity of our proposed approach, we present two experimental scenarios: $(1)$ learning a $\delta$-coherent function, and $(2)$ learning a non-$\delta$-coherent fuzzy function. We assess the resulting coherence, provide logic-based explanation, and, for the non-coherent case, we extend the explaining functor producing coherent explanations. Consider the $\alpha$-booleanization function $\delta_{0.5}$ where $\alpha=0.5$. For the experiments, we consider the following target functions:
\begin{itemize}
    \item[(1)] $f^{(1)}\colon [0,1]^2 \rightarrow \{0,1\}$ as the boolean XOR function $f^{(1)}(x,y)=\delta(x)\oplus\delta(y)$;
    \item[(2)]  $f^{(2)}\colon [0,1]^2 \rightarrow [0,1]$ as the \L ukasiewicz t-conorm (fuzzy OR) $f^{(2)}(x,y)= \min(1,x+y)$.
\end{itemize}
The function $f^{(2)}$ is non-$\delta$-coherent in the region $T=\{(x,y)\mid x+y \geq 0.5, x\leq 0.5,y\leq0.5 \}\subset[0,1]^2$ (see the example in the \hyperref[sec:introduction]{Introduction}). %For example, with inputs $x=y=0.4$, we have $\delta(0.4,0.4)=(0,0)=\delta(0,0)$, while $\delta\circ f^{(2)}(0.4,0.4)=\delta(0.8)\neq\delta(0)=\delta\circ f^{(2)}(0,0)$.

Both functions are implemented with a Logic Explained Network (LEN)~\cite{ciravegna2023logic} trained on synthetic datasets of uniformly generated points $(x, y)\in[0,1]^2$, comprising $1000$ training and $250$ validation samples for binary classification $(l\in\{0,1\})$. Performance evaluation uses a test set of $1000$ points strategically concentrated in $(1)$ around the central decision boundary or $(2)$ around the region~$T$ (see \cref{fig:datasets}).

\begin{figure}[ht]
\vspace{-15pt} 
    \centering
    \begin{subfigure}{0.4\textwidth}
        \centering
        \includegraphics[width=\textwidth]{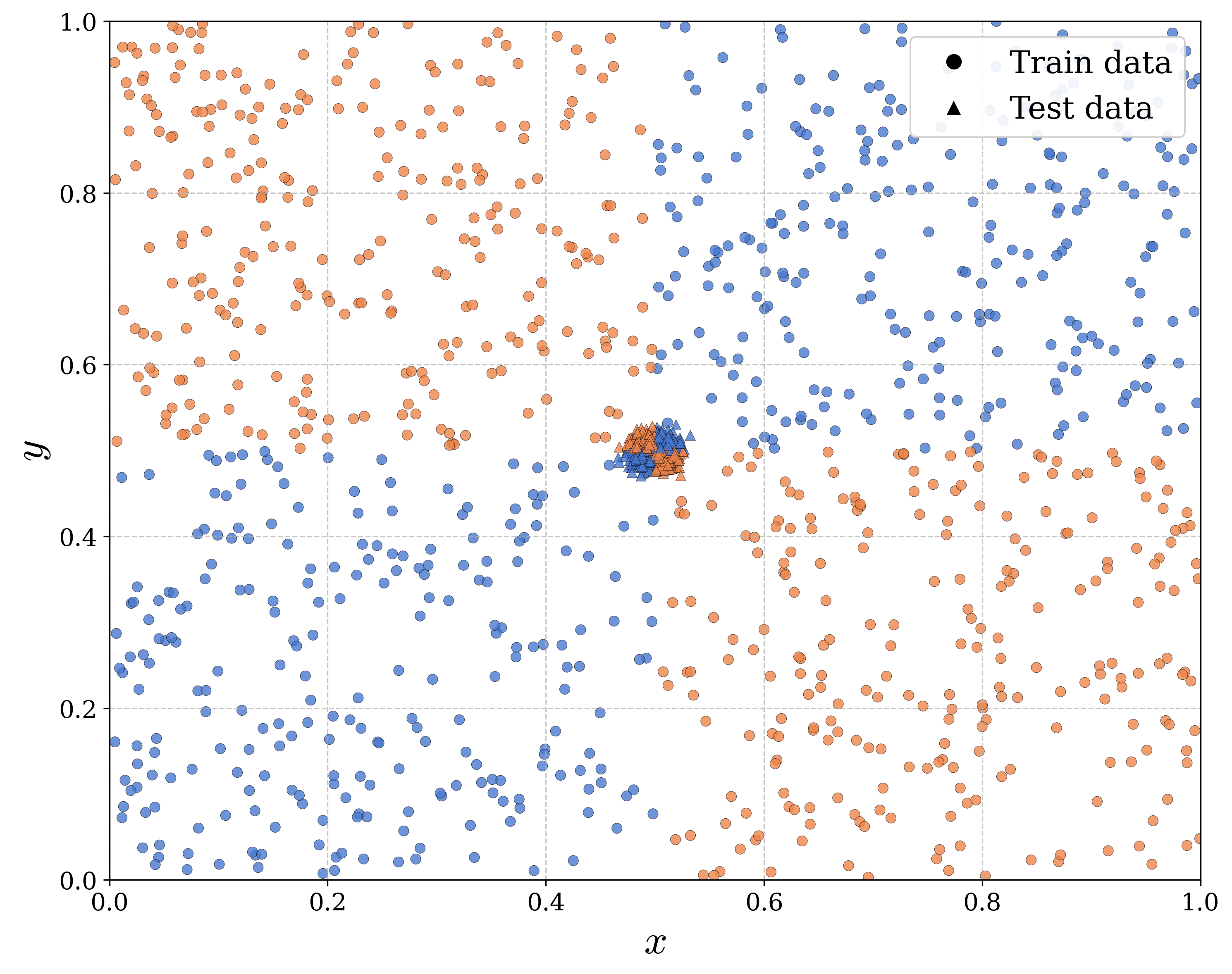}
    \end{subfigure}
    \begin{subfigure}{0.4\textwidth}
        \centering
        \includegraphics[width=\textwidth]{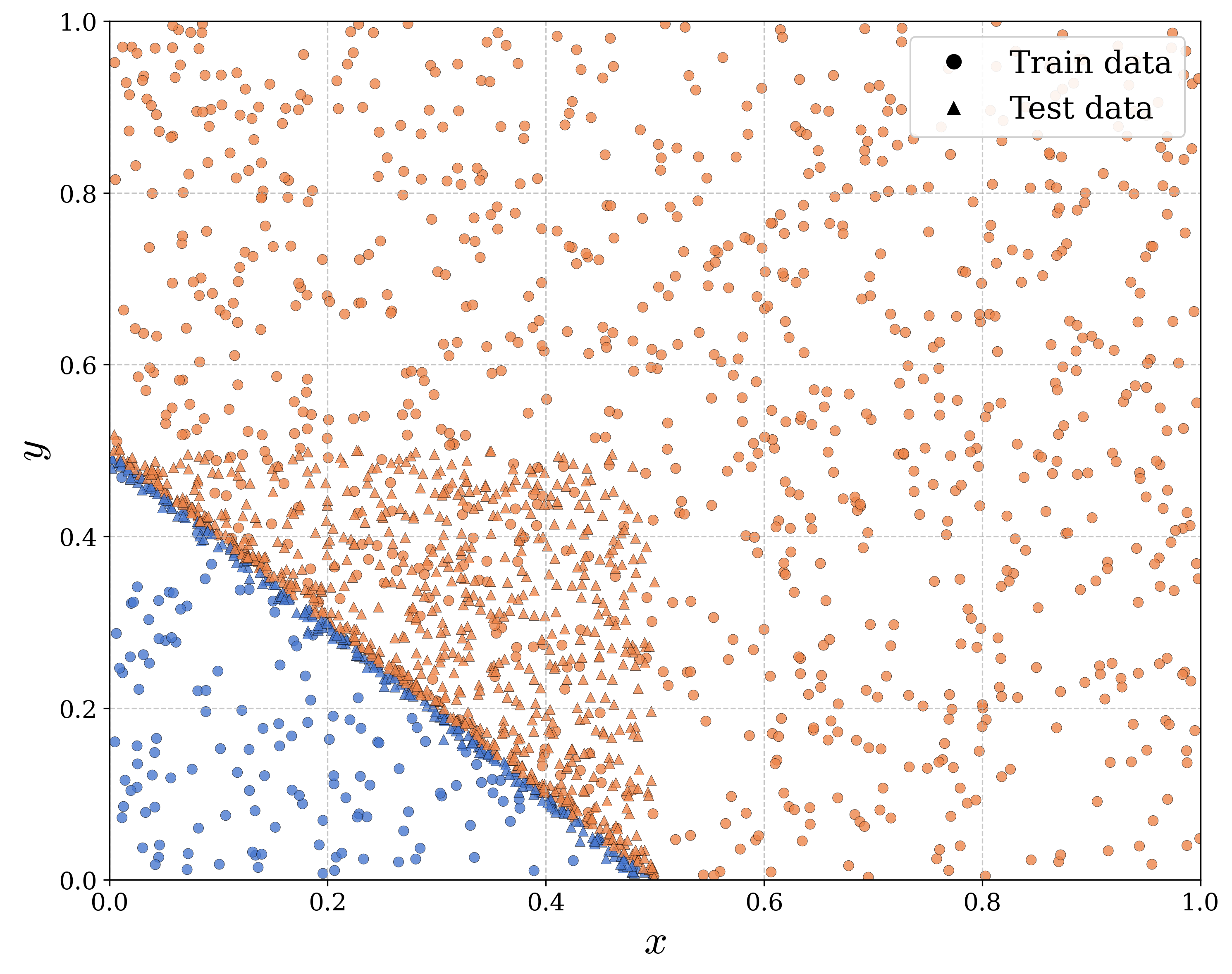}
    \end{subfigure}
    \caption{Data distributions for the XOR (left) and fuzzy OR (right) experiments, with colors indicating class labels.}
    \label{fig:datasets}
    \vspace{-15pt} 
\end{figure}
LENs predict output classes while providing First-Order Logic (FOL) explanations in terms of their inputs. In experiment $(1)$, the inputs are concepts computed by an embedding layer, and we force the coherence between the Boolean concepts and the labels by means of an additional regularization term added to the loss function. On the other hand, in $(2)$ we aim to learn a non-$\delta$-coherent function, so concepts come directly from the training data and there is no extra regularization term. We used PReLU activation functions between layers, with hyperparameters (number and size of hidden layers, learning rate, weight decay, coherence regularization multiplier) optimized through random search for best validation accuracy with early stopping. \Cref{tab:exp} reports the results grouped for both experiments, including accuracy, coherenceny between concepts and labels, FOL explanations, and fidelity. For $(2)$, we extended the explaining functor by applying \Cref{teo:extcoh} and \Cref{teo:cohgi}, yielding a post-hoc modified explainer $\hat{f}^{(2)}$ that discerns non-coherent samples by adding an additional feature~$nc\in \{0,1\}$.

\begin{table}[htb]
%\vspace{-25pt} 
\centering
\caption{Experimental results and explanations on the test set for $(1)$ (XOR) and $(2)$ (fuzzy OR) experimental settings. Values are in percentages over all corresponding data points. Explanations are computed on the test set.}
\label{tab:exp}
\resizebox{1\textwidth}{!}{\begin{minipage}{\textwidth}
\setlength{\tabcolsep}{3pt} % Adjust spacing between columns
\begin{tabularx}{\textwidth}{@{}l c *{6}{>{\centering\arraybackslash}X}ccc@{}}
\toprule
\multicolumn{2}{c}{\multirow{2}{*}{Exp.}} & \multicolumn{3}{c}{Accuracy} & \multicolumn{3}{c}{Coherency} & \multicolumn{3}{c}{Explanations} \\ 
\cmidrule(lr){3-5}\cmidrule(lr){6-8}\cmidrule(l){9-11}
\multicolumn{2}{c}{} & Train & Val. & Test & Train & Val. & Test & $l$ & FOL & Fidelity \\ 
\midrule
\multirow{2}{*}{(1)} & \multirow{2}{*}{$f^{(1)}$} 
  & \multirow{2}{*}{100} & \multirow{2}{*}{100} & \multirow{2}{*}{95.5} 
  & \multirow{2}{*}{100} & \multirow{2}{*}{100} & \multirow{2}{*}{94.8} 
  & $0$ & $(x \land y) \lor (\neg x \land \neg y)$ & 94.8 \\
  &  &  &  &  &  &  &  & $1$ & $(x \land \neg y) \lor (y \land \neg x)$ & 94.8 \\
\cmidrule(lr){1-11}
\multirow{4}{*}{(2)} & \multirow{2}{*}{$f^{(2)}$} & \multirow{4}{*}{99.9} & \multirow{4}{*}{99.2} & \multirow{4}{*}{88.4} & \multirow{4}{*}{98.5} & \multirow{4}{*}{97.6} & \multirow{4}{*}{67.1} & 0 & $\neg y$ & 67.1 \\
  &  &  &  &  &  &  &  & 1 & $x \lor \neg x = 1$ & 75.7 \\\cmidrule(r){9-11}
  & \multirow{2}{*}{$\hat{f}^{(2)}$} &  &  &  &  &  &  & 0 & $\neg x \land \neg nc$ & 83.8 \\
  &  &  &  &  &  &  &  & 1 & $x \lor nc$ & 83.8 \\
\bottomrule
\end{tabularx}\end{minipage}}
\vspace{-10pt} 
\end{table}

The experiments reveal that when the underlying function is naturally $\delta$-coherent, as in $(1)$, a LEN can achieve near-perfect accuracy and generates logically consistent, high-fidelity FOL explanations. In contrast, the fuzzy OR function in $(2)$ is inherently non-$\delta$-coherent, leading to a marked drop in explanation fidelity despite maintaining high accuracy. Importantly, the introduction of the extended explaining functor effectively compensates for non-coherence, significantly improving the quality and reliability of the explanations in the challenging region. This demonstrates that, while regularization during training can ensure coherence for certain neural functions, post-hoc adjustments are crucial for maintaining interpretability when dealing with non-coherent tasks.

\section{Related Work}

\paragraph{Rule-based Explainers.}
Our proposed framework aligns with and extends previous approaches to interpretable machine learning by ensuring logical consistency across explanations at multiple levels of abstraction. For example, LORE \cite{guidotti2024stable} constructs local interpretable predictors and derives decision and counterfactual rules, offering instance-specific explanations. \cite{Letham2015} introduce a generative model called Bayesian Rule Lists that yields a posterior distribution over possible decision lists. ANCHORS \cite{Ribeiro2018AnchorsHM} provide local explanations based on if-then logical rules.
Our explaining functor differs from these approaches by treating explanations as morphisms that preserve logical entailment across different levels of abstraction. For instance, we may ensure that local explanations remain logically coherent within a broader explanatory structure.
LENs \cite{ciravegna2023logic} are one of the most related approaches to our framework, as they are neural networks almost interpretable-by-design that can be easily converted into Boolean logic formulas. However, LENs explanations were not perfectly coherent, yielding to a limited fidelity of the explanations w.r.t. the original networks. In this sense, the explaining functor yields a practical advantage, maintaining logical consistency while mapping between explanatory levels. 

\paragraph{Categorical Foundations in (X)AI.} 
The category-theoretic perspective on machine learning has been increasingly adopted in recent years, see \cite{shiebler2021category} for a survey. Particularly close in spirit to our work are approaches that use functoriality to enable compositional analysis of machine learning procedures~\cite{cat2024}. These works focus on backpropagation and gradient-based learning. Our approach leverages similar categorical methods, but for a different purpose, as it focuses on (logic-based) explainability. Two recent categorical frameworks related with interpretability are~\cite{TullInterpretability} and~\cite{rodatz2024patternlanguagemachinelearning}. %(but see also~\cite{duneau2024scalableinterpretablequantumnatural} for compositional interpretability in the context of quantum processes). 
The work \cite{rodatz2024patternlanguagemachinelearning} models a different interpretation notion from ours, whereas~\cite{TullInterpretability} is concerned with defining elementary structures for interpretation at a higher level of abstraction. We expect the maturation of our frameworks to lead to a convergence of approaches.

% A promising avenue for future work stems from the observation that the category $\mc{B}$ of Boolean functions has a presentation in terms of generators and equations~\cite{LAFONT2003257}. Could we use such presentation to introduce a finer-grained notion of compositionality, which does not simply apply to composition of Boolean functions, but also leverages their decomposition in terms of elementary generators? A related question is whether our explaining functor is compatible with gradient-based learning of Boolean functions, as studied in~\cite{WilsonZanasi2020}. The fact that our work and~\cite{WilsonZanasi2020} are both formulated in a categorical language should facilitate a comparison. Finally, one of the benefits of a categorical approach is that our methodology may be probed on a different class of models, by simply switching category. There are two natural directions for future work: first, replace Boolean functions with the more expressive polynomial functions, which still support a notion of gradient-based learning~\cite{WilsonZ23}; second, replacing $[0,1]$ with a different quantale~\cite{Lawveremetric1973}, in order to study explainability for a different notion of quantitative process.

\section{Conclusions and Future Work}

In this paper, we have demonstrated that our categorical framework provides a mathematically grounded and conceptually coherent approach to self-explainable learning schemes. By leveraging Category Theory, we have shown that our framework enables the modeling of existing learning architectures to ensure both explainability and compositional coherence. Moreover, our approach forces cascade explainers to maintain structural consistency, thereby enhancing the interpretability of complex learning pipelines. Crucially, we introduced a novel theory for self-explainable learning schemes and provided rigorous mathematical proofs to establish their correctness and general applicability.

These findings open several promising directions for future research. One key avenue is the empirical validation of our categorical formulation of booleanized explanations over concept-based models, especially when dealing with sub-symbolic data like images. Moreover, we are convinced that our framework can be extended to support other types of explanation, such as LIME-like methods or saliency maps. This extension could allow the investigation on how apparently incomparable XAI methods may be connected (or not) by categorical functors. 
% Further exploration into alternative functorial mappings from fuzzy to discrete representations could provide additional insights into the fidelity and robustness of our approach. Additionally, optimizing neural networks to learn the best discrete explanations while preserving functorial properties presents an exciting challenge. Extending our framework to support other types of explanation, such as LIME-like methods or saliency maps, could further enhance its versatility. %Lastly, the connection between permutation-invariant functions and quotient categories suggests deeper structural properties worth investigating.
We think that our work may establish a principled foundation for self-explainability in learning systems, bridging theoretical rigor with practical interpretability. Future research in this direction has the potential to advance both theoretical understanding and real-world applicability of explainable AI models.

\bibliographystyle{splncs04}
\bibliography{references}

\end{document}